\documentclass[twoside]{article}

\usepackage[accepted]{aistats2022}
\usepackage[round]{natbib}

\usepackage{graphicx}
\usepackage{subfigure}

\usepackage{amsmath}
\usepackage{amsfonts}
\usepackage{amssymb}
\usepackage{amsthm}
\usepackage{mathabx}

\usepackage{wrapfig}
\usepackage{comment}

\usepackage{graphicx}
\usepackage{color}

\usepackage{algorithmicx}
\usepackage[Algorithm,ruled]{algorithm}
\usepackage[noend]{algpseudocode}
\usepackage{dsfont}
\usepackage{bm}

\newcommand{\node}{t}
\newcommand{\ord}{d}

\newcommand{\BlackBox}{\rule{1.5ex}{1.5ex}}  

\newtheorem{theorem}{Theorem}
 
\newtheorem{proposition}[theorem]{Proposition}

\newtheorem{definition}[theorem]{Definition}

\newcommand\blfootnote[1]{%
  \begingroup
  \renewcommand\thefootnote{}\footnote{#1}%
  \addtocounter{footnote}{-1}%
  \endgroup
}

\begin{document}

%

%

\twocolumn[

\aistatstitle{Deep Non-Crossing Quantiles through the Partial Derivative}

\aistatsauthor{ Axel Brando \And Joan Gimeno \And {\footnotesize Jos\'e A. Rodr\'iguez-Serrano} \And Jordi Vitri\`a }

\aistatsaddress{{UB$^{1}$ \& BSC$^{2}$} \And UB$^{1}$ \And BBVA$^{3}$ \And UB$^{1}$} ] 

\begin{abstract}
  Quantile Regression (QR) provides a way to approximate a single conditional quantile. To have a more informative description of the conditional distribution, QR can be merged with deep learning techniques to simultaneously estimate multiple quantiles. However, the minimisation of the QR-loss function does not guarantee non-crossing quantiles, which affects the validity of such predictions and introduces a critical issue in certain scenarios. In this article, we propose a generic deep learning algorithm for predicting an arbitrary number of quantiles that ensures the quantile monotonicity constraint up to the machine precision and maintains its modelling performance with respect to alternative models. The presented method is evaluated over several real-world datasets obtaining state-of-the-art results as well as showing that it scales to large-size data sets.
\end{abstract}

\section{INTRODUCTION}\label{sec:introduction}

\begin{figure}[ht]
\centering
 \includegraphics[scale=.6]{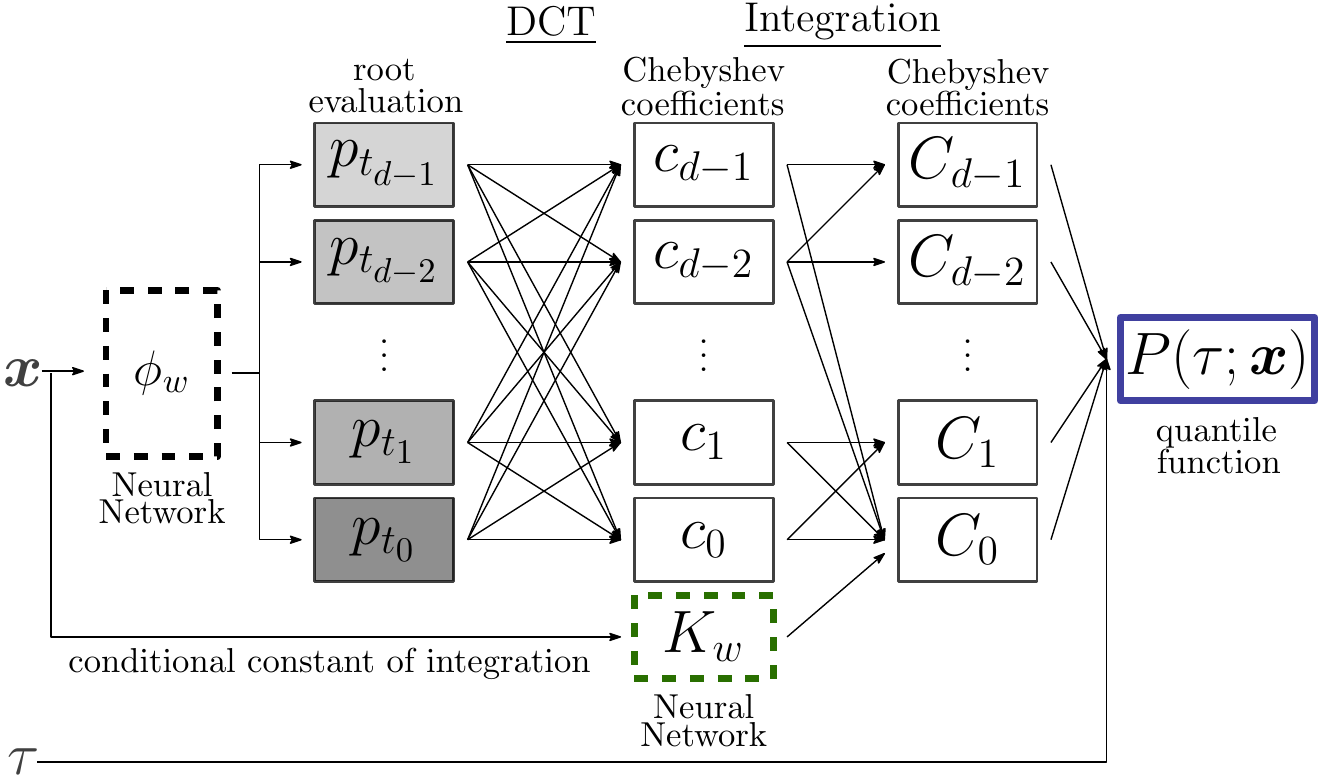}
 \includegraphics[scale=.43]{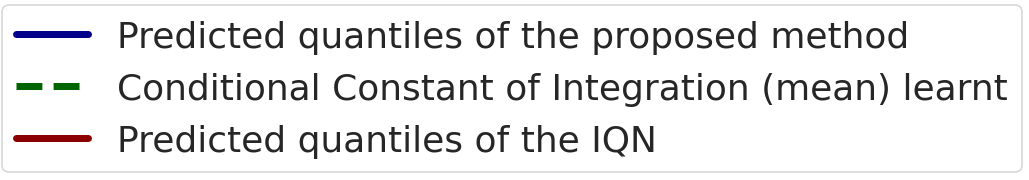}
 \includegraphics[scale=.46]{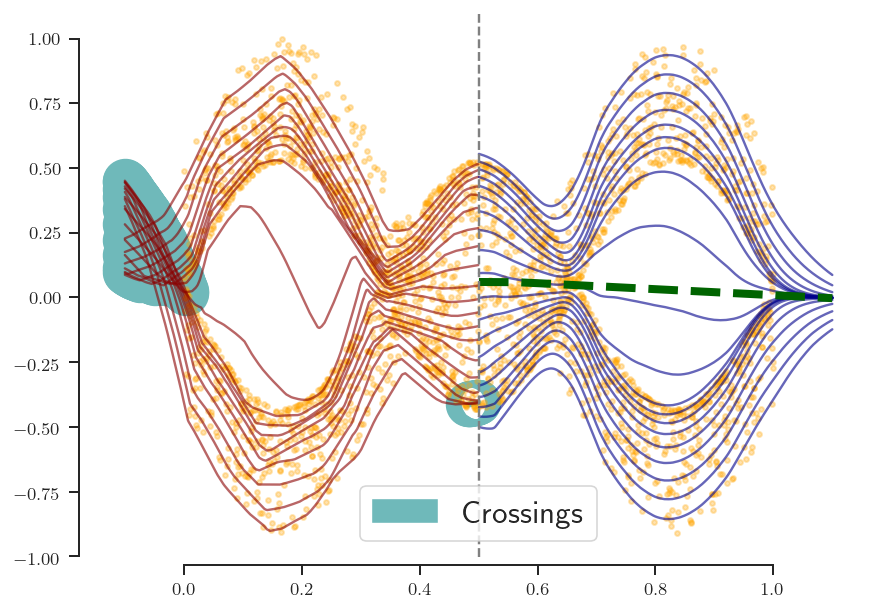}

\caption{{\small Synthetic data set showing that our proposed method (with the above architecture) avoids crossing quantiles (right), unlike IQN (left).}}
\label{fig:InitialFigure}
\end{figure}

Quantile Regression (QR) allows us to approximate a desired quantile -- unlike the classical regression that only estimates the mean or the median -- of the conditional distribution $p (Y \mid \boldsymbol{X})$. This is useful since we can capture confidence intervals without making strong assumptions about the distribution function to approximate. The formal definition of QR is: \blfootnote{${}^1$ Universitat de Barcelona (UB).} \blfootnote{${}^2$ Barcelona Supercomputing Center (BSC).} \blfootnote{${}^3$ BBVA AI Factory (BBVA).}

\begin{definition}[Conditional quantile regression]
Let $\boldsymbol{X} \in \mathbb{R}^D$ and $Y \in \mathbb{R}$ be respectively a covariate and a response random variable. Given $\tau$ in the real interval $[0,1]$, the conditional quantile regression (QR) consists in finding a function $q_\tau \colon \mathbb{R}^D \to \mathbb{R}$ which approximates the $\tau$-th quantile of $p (Y \mid \boldsymbol{X})$ by minimising the $\tau$-th quantile regression loss function defined as
\begin{multline}
 \label{eq:CQR}
 \mathcal{L}_{QR}(\boldsymbol{X}, Y, \tau) = \mathbb{E} \Big[ 
 \Big( Y - q_\tau(\boldsymbol{X}) \Big) \cdot \\
 \Big(\tau - \mathds{1}{[Y<q_\tau(\boldsymbol{X})]} \Big) \Big],
\end{multline}
where $\mathds{1}{\left[c\right]}$ denotes the indicator function that verifies the condition $c$.
\end{definition}

\begin{figure*}[t!]
\centering
\includegraphics[width=0.9\linewidth]{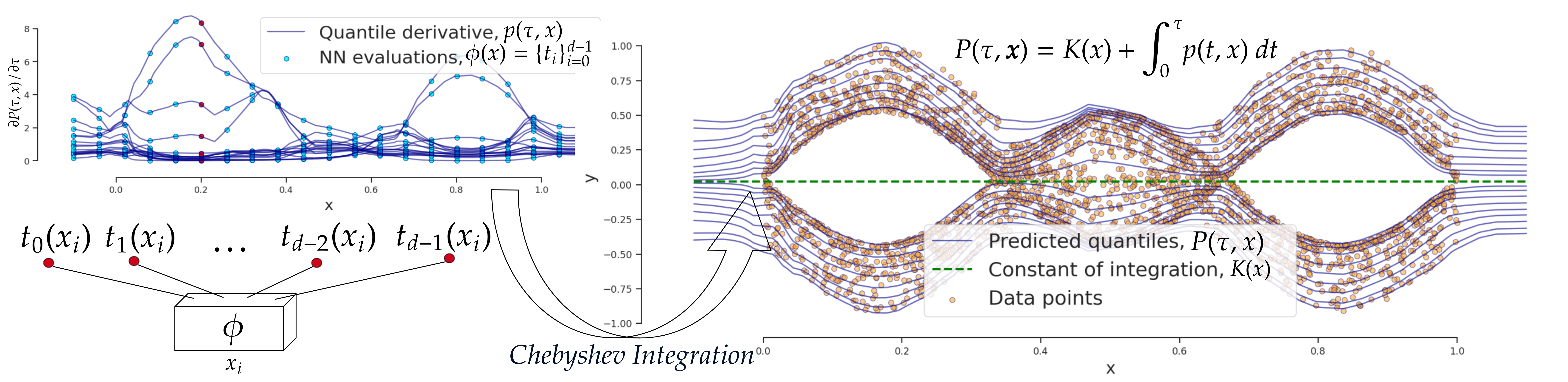}%
\caption{ In general terms, the proposed method uses a NN - $\phi$ at the bottom left - to generate $d$ \textbf{positive values} --the red $\{t_j(\boldsymbol{x})\}_{j=0}^{d-1}$ points--, which will be used as roots for computing the coefficients of a Chebyshev polynomial, represented in the top left subfigure. This polynomial will approximate the partial derivative of the quantile function. To do that, we integrate it obtaining a new Chebyshev polynomial, the one in the right subfigure. Thus, for each $\boldsymbol{x}$ we have a Chebyshev polynomial modelling all the quantiles.}
\label{fig:illustration}
\end{figure*}

Generically, the loss function in Eq.~\ref{eq:CQR} is asymmetric, convex, and it penalises overestimation errors with weight $\tau$ and underestimation errors with weight $1-\tau$.

Based on this loss function, QR models such as quantile regression forests \citep{meinshausen2006quantile}, gradient boosted quantile regression models \citep{zhang2018parallel} or even neural network models \citep{dabney2018distributional} can be designed to approximate a discrete set of quantiles. 

In all these cases the set of values of each $\tau$ is fixed a priori. Thus, the number of quantiles to approximate in a QR scenario constitutes a hyperparameter of the model definition and does not allow to approximate the full quantile function.

Alternatively, following \citet{dabney2018implicit, tagasovska2019single, brando2019modelling}, if the models are trained stochastically (such as the neural networks which use stochastic gradient descent), they can be defined to learn implicitly the quantile function, i.e. where the quantile to predict is an input parameter, $\tau$. As \citet{dabney2018implicit} states, this approach allows to approximate any conditional distribution given sufficient model capacity. In that case, we can extend Eq.~\ref{eq:CQR} to learning the full quantile distribution as follows:

\begin{definition}[Conditional quantile function]
A function $\Phi_w\colon [0,1]\times \mathbb{R}^D \rightarrow \mathbb{R}$ with parameters $w$ approximates the quantile function when it minimises the conditional quantile regression loss function defined as 
\begin{multline}
 \label{eq:CQF}
 \mathcal{L}(\boldsymbol{X},Y)= \mathbb{E} \Bigg[ \int_0^1 \Big(Y-\Phi _w(\tau,\boldsymbol{X}) \Big) \cdot \\ 
 \Big(\tau-\mathds{1}{[Y<\Phi _w(\tau,\boldsymbol{X})]}\Big) \, d\tau \Bigg],
\end{multline}
where $\mathds{1}{\left[c\right]}$ denotes the indicator function that verifies the condition $c$.
\end{definition}
The integral in Eq.~\ref{eq:CQF} is numerically hard to compute because the analytical expression of $\Phi _w$ is, generically, not known. Following \citet{brando2019modelling}, we can apply a Monte-Carlo strategy to provide a feasible loss function. This is based on considering a uniform random variable $\tau \sim \mathcal{U}(0,1)$ and for each evaluation of the loss function in Eq.~\ref{eq:CQF}, a $\tau$-sample set, $\{ \tau _t \} _{t=1}^{N_\tau}$, of $N _\tau$ points is generated in each training iteration such that 
\begin{multline}
 \label{eq:IQRLoss}
 \mathcal{L}(\boldsymbol{X},Y) \approx  \mathbb{E} \Bigg[ \frac{1}{N_{\tau}}\overset{N_{\tau}}{\underset{t=1}{\sum}} \Big(Y-\Phi_w(\tau_{t},\boldsymbol{X})\Big)\cdot\\  \Big(\tau_{t}-\mathds{1}{[Y<\Phi_w(\tau_{t},\boldsymbol{X})]}\Big) \Bigg].
\end{multline}

As we highlighted, the $\Phi _w$ requires to be a model that can be trained by using a Monte-Carlo approach. Hence, neural networks models will be suitable.

The modellization of quantile functions (shown in Eq.~\ref{eq:CQF}) constitutes an important goal, especially when the distribution to be predicted, $p(Y\mid \boldsymbol{X})$, is complex or when we are interested to impose the minimum assumptions about the shape of the distribution. A common practice is to consider an exponential power distribution (e.g. the Normal or Laplace distributions). In such cases, it is assumed unimodality, symmetry, and lose critical information regarding the shape of the distribution. 

Importantly, by building a quantile function we are able to approximate - in a discrete manner but with arbitrary precision - the distribution $p(Y\mid \boldsymbol{X})$. Typically, quantiles are used to build confidence intervals. However, when these confidence intervals are tight or when we want to recover the distribution shape from the discretization, we might face with the critical problem described below.

Since the different predicted quantile values are estimated individually, they may not be ordered according to their quantile value, $\tau$. For instance, the predicted quantile $0.1$ may not be a value greater than the quantile $0.05$. As a consequence, we obtain an non-valid distribution of the response variable. That phenomenon is known as the \emph{crossing quantile} phenomenon \citep{koenker2017handbook}: 

\begin{definition}[Crossing quantile phenomenon]
Let $f \colon [0,1] \times \mathbb{R}^D \rightarrow \mathbb{R}$ be a conditional quantile function that predicts a certain quantile of a random variable $y\in \mathbb{R}$ given $\boldsymbol{x}\in \mathbb{R}^D$. If there exist $\tau_1,\tau_2\in [0,1]$ being $\tau _1 < \tau _2$ and $\boldsymbol{x} \in \mathbb{R}^D$ such that 
\[
f(\tau_1,\boldsymbol{x}) > f(\tau_2,\boldsymbol{x}),
\]
then $f$ suffers the \emph{crossing quantile} phenomenon.
\end{definition}

Different solutions have been proposed to overcome this phenomenon. Most of them are based on adding a penalisation term to regularise the optimisation process and ``encourage'' the model to reduce the number of crossing quantiles \citet{koenker2001quantile, bondell2010noncrossing, tagasovska2018frequentist}. Nevertheless, since the model is not restricted to be {\bf partially monotonic} - i.e. monotonic with respect to $\tau$ and not over the other inputs - some quantiles may still cross.

\section{IMPOSING A POSITIVE PARTIAL DERIVATIVE}\label{sec:imposepositive}

Following the idea of partial monotonicity, we want to estimate a partial derivative and use its integral as a quantile function. 

\begin{definition}[Partial Derivative of the Conditional Quantile Function] \label{def:PDCQF}
Let $\Phi _{w} \colon [0,1] \times \mathbb{R}^D \to \mathbb{R}$ be a quantile function with parameters $w$ and let $\phi _w \colon [0,1] \times \mathbb{R}^D \to \mathbb{R} _+$ and $K _w \colon \mathbb{R} ^D \to \mathbb{R}$ be two functions with parameters $w$. If we assume that 
\begin{equation} \label{eq.Phiw}
    \Phi_{w}(\tau, \boldsymbol{x}) = K_w(\boldsymbol{x}) +  \int _0^\tau \phi_w(t, \boldsymbol{x})\, dt,
\end{equation}
then the conditional quantile function can be approximated by minimizing the following loss function,
{\small\begin{multline}
 \label{eq:DCQF}
 \mathcal{L}(\boldsymbol{X},Y)=  \mathbb{E} \Bigg[ \int_0^1  \left(Y - \Big( K_w(\boldsymbol{X}) +  \int _0^\tau \phi_w(t, \boldsymbol{X})\, dt \Big) \right) \cdot \\
  \left(\tau-\mathds{1}{\left[Y< \Big( K_w(\boldsymbol{X}) +  \int _0^\tau \phi_w(t, \boldsymbol{X})\, dt \Big)\right]} \right) \, d\tau \Bigg],
\end{multline}}
where $\mathds{1}{\left[c\right]}$ denotes the indicator function that verifies the condition $c$.
\end{definition}

Applying the approach in Definition~\ref{def:PDCQF}, we are modelling the partial derivative with respect to the quantile function, i.e. $\partial \Phi(\tau, \boldsymbol{x}) / \partial \tau =: \phi _w(\tau, \boldsymbol{x})$. Thus, we can impose conditions to that derivative. In our case, we want to impose that it will always be strictly positive to ensure the partial monotonic property of the predicted quantile function, i.e. $\forall \tau_1, \tau_2 \in [0,1]$, if~$\tau_1 < \tau_2$ then $\Phi(\tau_1, \boldsymbol{x}) < \Phi(\tau_2, \boldsymbol{x})$.

\subsection{Model Definition}

Following the Definition~\ref{def:PDCQF}, the main contribution of this article considers a conditional quantile function $\Phi _w\colon [0,1] \times \mathbb{R}^D \rightarrow \mathbb{R}$ that it is obtained by integrating its partial derivative $\phi _w$ and considering the constant of integration, $K _w(\boldsymbol{x})$. Generically speaking, $K _w(\boldsymbol{x})$ will manage the modelling of the conditional mean or the quantile $0$ of $p(Y\mid \boldsymbol{X})$ and $\phi _w(\tau,\boldsymbol{x})$ will approximate the rest of the $p(Y\mid \boldsymbol{X})$ distribution as quantiles of a new $q(Y\mid \boldsymbol{X})$.

To ensure the partial monotonicity of $\Phi _w$, a neural network $\phi _w \colon \mathbb{R}^D \rightarrow \mathbb{R}^{+}\times\overset{\ord}{\cdots}\times\mathbb{R}^{+}$ with $\ord$ positive outputs (as it is shown in Figure~\ref{fig:InitialFigure}) will be considered as the partial derivative of $\Phi _w$. Thus, the neural network $\phi _w$ approximates the derivative of the final quantile function with respect to the quantile $\tau$. 

The integration process of $\phi _w$ will be done by using truncated Chebyshev polynomial expansion of order $\ord$. Due to the truncation, we consider a finite mesh of quantile values, called \emph{Chebyshev roots} or \emph{roots}, $\{t_k\} _{k = 0}^{\ord-1} \subset [0,1]$ which are defined as
\begin{equation}
 \label{eq.nodes-values}
 \node _k = \frac{1}{2}\cos \biggl(\frac{\pi (k + \frac{1}{2})}{\ord}\biggr) +
 \frac{1}{2}, \quad 0 \leq k < \ord.
\end{equation} 
These Chebyshev roots only depend on the degree $\ord$ and not on the quantile input. Therefore, the roots are fixed once the degree is chosen\footnote{The non-dependency of the quantile value will be crucial to ensure the monotonicity of the global predicted quantile function up to the machine precision.}.

The truncated Chebyshev expansion consists in expressing a function as a linear combination of Chebyshev polynomials. These polynomials are defined as mappings $T _k \colon [-1,1] \rightarrow \mathbb{R}$ given by the recurrent formula
\begin{equation}
\label{eq.cheb-poly}
    \begin{split}
        T _0(t) &\coloneq 1, \\
        T _1(t) &\coloneq t, \\
        T _{k +1} (t) &\coloneq 2 t T _k(t) - T _{k-1}(t), \quad k \geq 1.
    \end{split}
\end{equation}
In our case, $\tau$ is in $[0,1]$ rather than in $[-1,1]$ which slightly modifies the standard polynomial definition and allows us to approximate $\phi _w (\tau, \boldsymbol{x})$ by 
\begin{equation}
 \label{eq.poly-p}
    p (\tau, \boldsymbol{x}; \ord) \coloneq \frac{1}{2}c _0(\boldsymbol{x}) +
 \sum _{k= 1}^{\ord -1} c _k(\boldsymbol{x}) T _k (2\tau-1).
\end{equation}
Like the roots, the coefficients $c _j(\boldsymbol{x})$ in Eq.~\ref{eq.cheb-poly} are independent of the quantiles. These quantities are computed for $j = 0, \dotsc, \ord -1$, by 
\begin{equation}
\label{eq.smallc}
\begin{split}
    c _j(\boldsymbol{x}) \coloneq \frac{2}{\ord} \sum _{k = 0}^{\ord -1} \phi _w(t _k, \boldsymbol{x}) \cos \biggl( &\frac{j \pi (k + \frac{1}{2})}{\ord}\biggr), 
\end{split}
\end{equation}
Eq.~\ref{eq.smallc} is merely a matrix-vector product from the matrix of cosines and the vector of $\phi _w(t _k, \boldsymbol{x})$ that instead of a complexity $\Theta(\ord^2)$, it can be computed in \mbox{logarithmic} complexity, $\Theta(\ord \log \ord)$.

The polynomials $T _k$ in Eq.~\ref{eq.poly-p} do not need to be explicitly computed and, by construction of the coefficients $c _k(\boldsymbol{x})$ in Eq.~\ref{eq.smallc}, $p(\node _k, \boldsymbol{x};d)$ is ``equal to'' $\phi _w(\node _k, \boldsymbol{x})$ for all the Chebyshev roots, $\{t_k\} _{k = 0}^{\ord-1}$, in Eq.~\ref{eq.nodes-values}. Note that $\phi _w(t _k, \boldsymbol{x})$ are denoted by $p _{t _k}$ in Figure~\ref{fig:InitialFigure}. That equality must, in practice, be understood in terms of machine precision of the
numerical representation system, classically $\sim 10^{-16}$ in
double-precision or $\sim 10^{-8}$ in single-precision
arithmetic. The root evaluation step is illustrated in Figure~\ref{fig:InitialFigure} and its values are denoted as $p _{t _k}$.

Once $p(t,\boldsymbol{x}; \ord)$ has been encoded by its $c_j(\boldsymbol{x})$ coefficients in Eq.~\ref{eq.smallc}, we can easily compute its integral function $\int _0 ^\tau p(t, \boldsymbol{x}; \ord) \, dt$ denoted by $P(\tau, \boldsymbol{x}; \ord)$. Thus $P$ will be an approximation of the integral of the neural network, $\phi _w$. That is,
\begin{equation}
    \label{eq.P-Phi}
    P(\tau, \boldsymbol{x}; \ord) \approx \Phi _w(\tau, \boldsymbol{x}) = K _w(\boldsymbol{x}) + \int _0^\tau \phi _w(t, \boldsymbol{x})\, dt.
\end{equation}
Additionally, given that we imposed that the neural network $\phi _w(t,\boldsymbol{x})$ gives only positive values for all $t \in [0,1]$, then $P(\tau,
\boldsymbol{x}; \ord)$ would be an increasing function with respect to
$\tau$ as long as $\ord$ is large enough. 

In the above procedure, the integral of $\phi _w$ in Eq.~\ref{eq.P-Phi} is, in general, not straightforward to compute. 
However, the neural network $\phi _w(\cdot, \boldsymbol{x})$ is globally being represented by $p(\cdot, \boldsymbol{x}; \ord)$ on the quantile interval $[0,1]$. Therefore, by using its Chebyshev coefficients $c_k(\boldsymbol{x})$, we can encode the integral of $P$ to provide its Chebyshev coefficients $C_k(\boldsymbol{x})$. In fact, according
to \citet{Clenshaw1955}, the integral of $p(\tau, \boldsymbol{x}; \ord)$ gives another Chebyshev
expansion, say $P(\tau, \boldsymbol{x}; \ord)$, with coefficients
$C_k(\boldsymbol{x})$, i.e.,
\begin{equation}
 \label{eq.poly-P} 
 P(\tau, \boldsymbol{x}; \ord) = \frac{1}{2}C _0(\boldsymbol{x}) +
 \sum _{k = 1}^{\ord -1} C _k(\boldsymbol{x}) T _k (2\tau-1).
\end{equation}
To deduce the expressions for $C _k(\boldsymbol{x})$ in Eq.~\ref{eq.poly-P}, we need to recurrently integrate the
polynomials $T _k(t)$, whose integral are 
\begin{equation}
\begin{split}
 \int T _0 (t) \,dt &= T _1(t) + \text{constant}, \hspace{0.5cm} 
 \\
 \int T _1 (t) \,dt &= \frac{T _2(t) + T _0(t)}{4} + \text{constant}, \\ 
 \int T _k (t) \,dt &= \frac{T _{k-1}(t)}{2(k-1)}  - \frac{T 
_{k+1}(t)}{2(k+1)}  + \text{constant}, \quad k \geq 2. \label{eq.integralT}
\end{split}
\end{equation}
By ordering the coefficients $c _j(\boldsymbol{x})$ in Eq.~\ref{eq.poly-p} in terms of $C _j(\boldsymbol{x})$ of Eq.~\ref{eq.poly-P}, we deduce that
\begin{equation}
\label{eq.coeffsC}
\begin{split}
 & C _k(\boldsymbol{x}) \coloneq \frac{c _{k-1}(\boldsymbol{x}) - c
   _{k+1}(\boldsymbol{x})}{4 k}, \quad 0 < k < \ord -1,\\ 
 & C _{d-1}(\boldsymbol{x}) \coloneq \frac{c
   _{d-2}(\boldsymbol{x})}{4(d-1)},
\end{split}
\end{equation}
and $C _0(\boldsymbol{x})$ depends on the constant of integration
$K_w(\boldsymbol{x})$ in Eq.~\ref{eq.P-Phi} and the other coefficient
values in Eq.~\ref{eq.poly-P}.

On the whole, the proposed method consisting of the following steps: for each $\boldsymbol{x}$ value,
\begin{enumerate}
    \item A set of values, $\{ p_{t_k}(\boldsymbol{x})\}_{k=0}^{\ord-1}$ is obtained via the neural network $\phi _w$ at the corresponding roots, $\{ t_k \}_{k=0}^{\ord-1}$.
    \item These are transformed to $\ord$ coefficients, $\{ c_{k}(\boldsymbol{x})\}_{k=0}^{\ord-1}$ using a fast matrix-vector multiplication algorithm. This results in a truncated Chebyshev polynomial, $p(\tau,\boldsymbol{x};\ord)$.
    \item This polynomial is then integrated to obtain another Chebyshev polynomial, $P(\tau,\boldsymbol{x};\ord)$, whose coefficients are $\{ C_{k}(\boldsymbol{x})\}_{k=0}^{\ord-1}$ in Eq.~\ref{eq.coeffsC}.
\end{enumerate}
The final output is a Chebyshev polynomial for each $\boldsymbol{x}$ value that approximates the unknown conditional quantiles of $p(Y \mid \boldsymbol{X})$.

\subsection{Selecting the Constant of Integration}\label{subsec:ConstantChoice}

The formula used to calculate the constant of integration $C _0(\boldsymbol{x})$ (shown in Eq.~\ref{eq.coeffsC}) will depend on which statistic we choose for $K(\boldsymbol{x})$, as shown in the following propositions. 

\begin{proposition}[Our-$q _0$]
\label{prop:chenetq0}
Let $(P, K)$ be the proposed model and $q(Y\mid \boldsymbol{X})$ the distribution that $P$ produces as quantiles. If $K$ is the \textbf{lowest quantile} ($\tau=0$) of $q(Y\mid \boldsymbol{X})$, then the $C _0$ coefficient of the Chebyshev polynomial $P$ verifies:
\begin{equation}\label{CheNet:C0}
    C _0(\boldsymbol{x}) = 2 K(\boldsymbol{x}) - 2 \sum
    _{k=1}^{\ord-1} C _k(\boldsymbol{x}) (-1)^k.
\end{equation}
\end{proposition}
\begin{proof}
The learning process is subjected to the condition that the lowest quantile value, i.e. $\tau = 0$, must be the $K$. That is, $P(0, \boldsymbol{x}; d) = K(\boldsymbol{x})$.
By taking the value $t = -1$ in Eq.~\ref{eq.cheb-poly}, we derive that $T _k(-1) = (-1)^{k}$. Therefore, using Eq.~\ref{eq.poly-P} with $\tau=0$, we obtain Eq.~\ref{CheNet:C0}.
\end{proof}

\begin{proposition}[Our-Mean]
Let $(P, K)$ be the proposed model and $q(Y\mid \boldsymbol{X})$ the distribution that $P$ produces as quantiles. If $K$ is the \textbf{mean} of $q(Y\mid \boldsymbol{X})$, then the $C _0$ coefficient of the Chebyshev polynomial $P$ verifies:
\begin{equation}\label{CheNet:Mean}
    C _0(\boldsymbol{x}) = 2K(\boldsymbol{x}) - 2\sum _{\substack{k = 1 \\ k \text{ odd}}}^{d-1}
    \frac{C _k(\boldsymbol{x})}{k^2 - 4}.
\end{equation}
\end{proposition}
\begin{proof}
Taking into account the definition of the (continous) mean, the condition we must impose is 
\begin{equation}
    \int _0 ^1 \tau P (\tau, \boldsymbol{x}; d) \, d\tau = K(\boldsymbol{x}).
\end{equation}
Then, taking into account the linearity of the integral and after a change of coordinates ($t = 2\tau -1$), the integral is reduced to compute the mean of the Chebyshev polynomials. That is, 
 \begin{equation}
  \label{eq.integralmeanT}
     \int _{-1}^1 t T_k(t) \, dt.
 \end{equation}
Taking into account the symmetries of the Chebyshev polynomials, we obtain the equality
 \begin{equation*}
     \int _{-1}^1 t T_k(t) \, dt = (1 + (-1)^{k+1}) \int _{0}^1 t T _k(t)\, dt.
 \end{equation*}
Then if $k$ is even, Eq.~\ref{eq.integralmeanT} is zero. If $k$ is now assumed to be an odd integer, then taking into account the recurrent definition of the Chebyshev polynomials, $2 T _k(t) = T _{k+1}(t) + T _{k-1}(t)$ and then Eq.~\ref{eq.integralT}, 
 \begin{equation*}
   \int _{-1}^1 t T_k(t) \, dt = \frac{T _{k-2}(t)}{2(k-2)} - \frac{T
     _{k+2}(t)}{2(k+2)} \biggr|_{0}^1 = \frac{2}{k^2-4}.
 \end{equation*}
 From here, we can recover the expression in Eq.~\ref{CheNet:Mean}.
\end{proof}

\subsection{Rate of Convergence of the Chebyshev Expansion and Machine Precision} \label{subsect:Rate}

In Chebyshev series theory \citep{Trefethen2008,Majidian2017}, if a function $\phi _w$ is of class $C ^1$ in $[-1,1]$, then its Chebyshev expansion converges absolutely and uniformly to $\phi _w$ in $[-1,1]$. In case that $\phi _w$ is in $C^k([-1,1])$, then its Chebyshev coefficients $c _\ell$ verify that $| c _\ell | = o (1/\ell^k)$. Moreover, if $\phi _w$ is now analytic, its Chebyshev expansion will have an exponential rate, i.e. $|c _ \ell | = o(\exp(- \rho \ell))$ for some $\rho > 0$ denoting the radius strip in the complex plane whose strip contains all the Chebyshev coefficients $c _ \ell$. Therefore, smoother mappings are going to require less Cheybshev coefficients. To have a kind of measure about the accuracy of the approximation $p\approx\phi$, one needs to check the decay rate of the Chebyshev coefficients $c _\ell$. Most of the times, it will be enough to monitor the last two coefficients $c _ {\ord -1} (\boldsymbol{x})$, and $c _ {\ord -2} (\boldsymbol{x})$ in absolute value.

Due to the discretisation in the root mesh $\{ t _k \} _{k=0}^{\ord -1}$, which acts as a linear transformation, the evaluation of $P(t _k; \boldsymbol{x})$ will have roundoff error (i.e. machine precision) with respect to the integral of $\int _0^ {t _k} \phi _w(t; \boldsymbol{x}) \, dt$ and for the other values not in the mesh, its (absolute) error will be bounded by the aforementioned convergence rate.

\subsection{Ensure Monotonicity for all Quantiles}\label{subsec:Mono}

There is one last detail to be totally confident that the function $P(\tau; \boldsymbol{x}) = f(\tau, \boldsymbol{x})$ is monotonic with respect to $\tau$. It is required to consider that $p$ is an approximation of $\phi _w$, i.e. $p(\tau; \boldsymbol{x}) \approx \phi _w(\tau; \boldsymbol{x})$. Thus, although the $\phi _w$ function has been forced to be strictly positive, the approached Chebyshev polynomial $p$ may not be. According to Section~\ref{subsect:Rate}, we are certain that for the roots values, $\{t_k \} _{k=0}^{\ord-1}$, the Chebyshev estimation has a negligible error. However, in the middle points between the roots it is important to be careful. 

Once the appropriate degree is known, the model with that degree can be considered a monotonic function. However, it is important to note that although the error will be small, the possibility of $P$ ceasing of being monotonic exists in all non-root points. In case that the order must be increase to ensure the monotonicity, the number of parameters to learn will increase as well.

\subsection{Calculation Procedure of Coefficients}

To evaluate Equation \ref{eq.poly-p} or Equation \ref{eq.poly-P} for a value $\tau$ in $[0,1]$, the Clenshaw's method \citep{Clenshaw1955} or its stable numerical error version \citep{Elliott1968,Newbery1974} works. Let us briefly summarise the evaluation at $\tau$ of $p(\tau; \boldsymbol{x})$ in  Equation \ref{eq.poly-p} (or $P(\tau; \boldsymbol{x})$ in Equation \ref{eq.poly-P}), whose numerical complexity is $\Theta(\ord)$.

\begin{algorithm}[!htb]
\caption{Evaluation of a Chebyshev sum at a given $\tau\in[0,1]$. In particular, useful for Eqs.~\ref{eq.poly-p} and \ref{eq.poly-P}.}
\label{algo.eval_cheb}
\begin{algorithmic}[1]
\Procedure{eval\_cheb}{$\tau$, $c _0(\boldsymbol{x}), \dotsc, c _{\ord -1}(\boldsymbol{x})$}
\State $d _1(\boldsymbol{x}) \gets d _2(\boldsymbol{x}) \gets d _3(\boldsymbol{x}) \gets 0$.
\State $\sigma \gets 2 \tau -1$.
\For {$k = \ord -1, \ord-2, \dotsc, 1$}
\State $d _3(\boldsymbol{x}) \gets d _1(\boldsymbol{x})$.
\State $d _1(\boldsymbol{x}) \gets 2 \sigma d _1(\boldsymbol{x}) - d _2(\boldsymbol{x}) + c _k(\boldsymbol{x})$.
\State $d _2(\boldsymbol{x}) \gets d _3(\boldsymbol{x})$.
\EndFor
 \Return $\sigma d _1(\boldsymbol{x}) - d _2(\boldsymbol{x}) + 0.5 c _0(\boldsymbol{x})$.
\EndProcedure
\end{algorithmic}
\end{algorithm}

\algnewcommand{\LineComment}[1]{\State \(\triangleright\) #1}

\begin{algorithm}[!htb]
\floatname{algorithm}{Prerequisites}
\caption{Definitions and functions used for next Algorithms.}
\label{algo:Prerequisites}
\begin{algorithmic}
\LineComment{$\boldsymbol{x}$ has batch size and number of features as shape, i.e. $\left[ bs, D\right]$.}
\LineComment{\texttt{RS}$\left( \; tensor, shape\right)$: reshape $tensor$ to $shape$. }
\LineComment{\texttt{RP}$\left(tensor, n\right)$: repeats $n$ times the last dimension of $tensor$.}
\LineComment{\texttt{CC}$\left( T_1, T_2\right)$: concatenate $T_1$ and $T_2$ by using their last dimension. }
\end{algorithmic}
\end{algorithm}

\begin{algorithm}[!htb]
\caption{Chebyshev coefficients of the integral of a non-negative neural network.}
\label{algo.coeff_cheb}
\begin{algorithmic}[1]
\Procedure{cheb\_cs}{$\boldsymbol{x}, \ord, \phi _w, K $}
\State $\{ o_k(\boldsymbol{x}) \}_{k=0}^{\ord -1} \gets \phi _w{\left( \boldsymbol{x} \right)}$\Comment{Apply any NN, $\phi _w$.}
\State $\{ c_k(\boldsymbol{x}) \}_{k=0}^{\ord -1} \gets \text{DCT-II}(\boldsymbol{o},\ord)$\Comment{Eq.~\ref{eq.smallc}}
\State $\{ C_k(\boldsymbol{x}) \}_{k=1}^{\ord -1} \gets$ Integration step wrt $\{ c_k(\boldsymbol{x}) \}_{k=0}^{\ord -1}$
\State $C_0(\boldsymbol{x}) \gets 2 K (\boldsymbol{x}) - 2 \sum _{k=1}^{\ord-1} C _k(\boldsymbol{x}) (-1)^k$ 
\Return $\{ C_k(\boldsymbol{x}) \}_{k=0}^{\ord -1}, \{ c_k(\boldsymbol{x}) \}_{k=0}^{\ord -1}$.
\EndProcedure
\end{algorithmic}
\end{algorithm}

\subsection{Chebyshev Polynomial as a Framework} \label{sub:framework}

\begin{algorithm}[!htb]
\caption{How to build the model by using any deep learning architecture for regression.}
\label{algo.build_chepan}
\begin{algorithmic}[1]
\Procedure{build\_model\_graph}{$\boldsymbol{x}, y, \ord, \phi , K, N_\tau$}
\LineComment{$N_\tau$ is the number of non-roots to evaluate.}
\State $\{ C_k(\boldsymbol{x}) \}_{k=0}^{\ord -1}, \{ c_k(\boldsymbol{x}) \}_{k=0}^{\ord -1} \gets \text{\texttt{CHEB\_CS}}(\boldsymbol{x}, \ord, \phi , K)$
\State $\boldsymbol{\tau} \gets \mathcal{U}{\left(0,1\right)}$ \Comment{$\boldsymbol{\tau}$ must has $\left[bs \cdot N_\tau, 1\right]$ shape.}
\State $\boldsymbol{o}_{P} \gets \text{\texttt{EVAL\_CHEB}}(\boldsymbol{\tau}, C _0(\boldsymbol{x}), \dotsc, C _{\ord -1}(\boldsymbol{x}))$
\State $\mathcal{L} \gets \left(y-\boldsymbol{o}_{P}\right)\cdot\left(\boldsymbol{\tau}-\mathds{1}{\left[y<\boldsymbol{o}_{P}\right]}\right)$\Comment{\text{Eq. \ref{eq:CQR}} loss.}
\State \textbf{return} $\mathcal{L}$
\EndProcedure
\end{algorithmic}
\end{algorithm}

On the whole, the proposed method can be seen as a framework to use any deep learning architecture to build a partial monotonic function. This monotonic function with respect some of the input variables could be applied to solve, with guarantees, the crossing quantile phenomenon of conditional quantile regression models. The implementation of this framework applying to quantile regression is shown in Algorithm~\ref{algo.build_chepan} and could be done with any automatic differentiation library such as TensorFlow \citep{abadi2016tensorflow} or PyTorch \citep{paszke2017automatic}. Additionally, this implementation takes care of performing the matrix-vector product efficiently by using the DCT-II referred previously. 

\subsection{The Desired Quantile as an Input}

\begin{algorithm}[!htb]
\caption{How to evaluate the model for any quantile $\tau\in\left[0,1\right]$ desired to obtain $P(\tau; \boldsymbol{x})$ and $p(\tau; \boldsymbol{x})$.}
\label{algo.eval_chepan}
\begin{algorithmic}[1]
\Procedure{eval\_model}{$\boldsymbol{x}, \ord, \phi _w, K, \tau$}
\State $\{ C_k(\boldsymbol{x}) \}_{k=0}^{\ord -1}, \{ c_k(\boldsymbol{x}) \}_{k=0}^{\ord -1} \gets \text{\texttt{CHEB\_CS}}(\boldsymbol{x}, \ord, \phi _w, K)$
\State $\boldsymbol{o}_{p} \gets \text{\texttt{EVAL\_CHEB}}(\tau, c _0(\boldsymbol{x}), \dotsc, c _{\ord -1}(\boldsymbol{x}))$
\State $\boldsymbol{o}_{P} \gets \text{\texttt{EVAL\_CHEB}}(\tau, C _0(\boldsymbol{x}), \dotsc, C _{\ord -1}(\boldsymbol{x}))$
\State \textbf{return} $\boldsymbol{o}_{P}, \boldsymbol{o}_{p}$
\EndProcedure
\end{algorithmic}
\end{algorithm}

It is important to highlight that even if the number of $\phi _w$ outputs is fixed in the proposed model, we are able to estimate any quantile as it is show in Algorithm~\ref{algo.eval_chepan}.

\subsection{Inverse Mapping}

For all $\boldsymbol{x} \in \mathbb{R}^{D}$, when the predicted function is monotonic, then $P(\cdot ; \boldsymbol{x})$ is bijective: Given $y\in\mathbb{R}$ in the image of a monotone $P(\cdot; \boldsymbol{x})$, we can get the unique quantile $\tau$, for some fixed $\boldsymbol{x}$ and $y$, such that $P(\tau; \boldsymbol{x}) = y$. To obtain $\tau^{\ast} \approx P^{–1}(y, \boldsymbol{x})$, we can proceed by an iterative scheme given by

\begin{equation}
 \tau _{n+1} = \tau _n - p(\tau _n, \boldsymbol{x})^{-1} \bigl( P(\tau _n; \boldsymbol{x}) - y \bigr), n \geq 1,
\end{equation}

where the value $p(\tau _n, \boldsymbol{x})$ as well as $P(\tau _n; \boldsymbol{x})$ can be calculated by Equation \ref{eq.poly-p}. 

\section{ABLATION STUDY}

To clarify the proposed method, in the following section we will propose a quasi non-crossing quantile model (not up to machine precision), which also estimates the partial derivative and the closest - as far as we know - literature model, but it cannot applicable as a quantile function estimator because it does not provide a uniform global behaviour in $[0,1]$.

\subsection{Non-Strictly Positive Partial Derivative}

Here we propose a slight variation of the proposed model, which we will denote as \textit{Not Always Monotonic} model (NAM).

Rather than considering the interval of roots fixed, which implies to fix the roots values, an alternative manner to compute the Eq.~\ref{def:PDCQF} integral can be to compute the Clenshaw-Curtis formula over the interval $[0,\tau]$. In that case, that formula, considering the quantile input, consists in the following steps:

\begin{enumerate}
\item First, we fix an \emph{even} integer $\ord$, called degree, and we define the so-called nodes depending on $\tau$:
\begin{equation}
 \label{eq.nodes-values-II}
 \bar \node _k^\ord(\tau) \coloneq  \frac{\tau}{2}\cos \biggl( \frac{\pi k}{\ord} \biggr) +
 \frac{\tau}{2}, \quad 0 \leq k \leq \ord.
\end{equation} 
These nodes have the property that $\bar \node _k^{2\ord}(\tau) = \bar \node _{k/2}^{\ord}(\tau)$, which means that half of them can be reused when the value of $\ord$ is doubled. 

\item Second, we compute the quantities $ \bar c _j (\tau, \boldsymbol{x})$ defined for $0 \leq j \leq \ord$ as 
\begin{equation}
\label{eq.smallbarc}
    \bar c _j (\tau, \boldsymbol{x}) \coloneq  \sum _{k = 0}^\ord \phi _w(\bar t _k^\ord(\tau), \boldsymbol{x}) \cos \biggl(\frac{j \pi k}{\ord}\biggr), 
\end{equation}
which is just a matrix-vector multiplication. Here, we can use an algorithm, such as the Discrete Cosine Transform of type 1 (DCT-I), which performs the matrix-vector multiplication in Eq.~\ref{eq.smallbarc} with a complexity $\Theta(\ord \log \ord)$ rather than a standard $\Theta(\ord^2)$ procedure. In general, this algorithm produces the unnormalized coefficients defined in Eq.~\ref{eq.smallbarc}. To normalize them with respect to the degree, we must apply a factor, for instance, $2/\ord$. Thus, let us redefine the quantities 
\[
 \bar c _j (\tau, \boldsymbol{x}) \gets\frac{2}{\ord} \bar c _j(\tau, \boldsymbol{x}) . 
\]
\item Finally, $\Phi _w(\tau ; \boldsymbol{x})$ in Def.~\ref{def:PDCQF} is approximated by $P(\tau; \boldsymbol{x}, \ord)$ such that $P(0, \boldsymbol{x}; \ord) = \Phi _w(0;\boldsymbol{x})$, i.e. the constant of integration in Def.~\ref{def:PDCQF} is the conditional quantile $\tau=0$. All of the above leads us to the final Clenshaw-Curtis expression used in the NAM,
\begin{equation}
\label{eq.clenshaw-curtis}
  P(\tau, \boldsymbol{x};d) = \tau \biggl( \frac{\bar c
    _0(\tau, \boldsymbol{x})}{2}  - \sum _{k = 1}^{ d /2}
  \frac{\bar c _{2k}(\tau, \boldsymbol{x})}{4k^2 - 1} \biggl) + \Phi _w(0;\boldsymbol{x}).
\end{equation}
\end{enumerate}

Note that Eq.~\ref{eq.clenshaw-curtis} has a $\tau$ dependency on all the coefficients of $\bar c _{k}(\tau, \boldsymbol{x})$, which comes from the fact that the $P$ depends on $\tau$, because the nodes in Eq.~\ref{eq.nodes-values-II} also originally depend on the quantile $\tau$. This dependency will avoid to have a certain values of quantiles while ours methods can ensure - in machine precision - that the crossing quantile phenomenon does not appear because it approximates $\phi _w$ independently on $\tau$.

\subsection{Related Work}

Building generic monotonic functions poses an important problem in several areas \citep{archer1993application,sill1998monotonic,daniels2010monotone,gupta2016monotonic,you2017deep,wang2020monotonic}. In the econometric field there are works such as \citep{KoenkerN05}, which precisely explore the incorporation of monotonicity constraints and the estimation of the whole quantile function at once. However, as \citep{FeldmanBR21} states, these models, such as the VQR in \citep{Carlier2020}, assume linearity, which in some contexts is a strong assumption. Differently, in the current approach, we are estimating the partial derivative of the quantile function by means of a truncated Chebyshev polynomial. Moreover, the proposed technique allows us to guarantee that the resulting quantile function is always partial monotonous in the desired quantiles even our proposal is using internally a deep learning model. Therefore, our proposal goes beyond modelling the conditional quantile function and not only constitutes a crossing quantile phenomenon solution but it allows us to estimate the partial derivative of the quantile function by means of a neural network.

\begin{table*}[t!]
\setlength\tabcolsep{1.5pt}
\caption{
Minimum and maximum, [min, max], of the crossing quantile numbers over all the test folds proposed in \citet{Hern}.
* denotes the number of crossing quantiles evaluating the roots quantiles.
  }
\begin{center}
\scalebox{0.85}{
\begin{tabular}{r|c|c|c|c|c|c|c|c|c|c|c}
\multicolumn{1}{r}{}
& \multicolumn{1}{c}{\hfil Housing}
& \multicolumn{1}{c}{\hfil Concrete}
& \multicolumn{1}{c}{\hfil Energy}
& \multicolumn{1}{c}{\hfil Kin8nm}
& \multicolumn{1}{c}{\hfil Naval}
& \multicolumn{1}{c}{\hfil Power}
& \multicolumn{1}{c}{\hfil Protein}
& \multicolumn{1}{c}{\hfil Wine}
& \multicolumn{1}{c}{\hfil Yacht}\\ \cline{2-10}
N & $\boldsymbol{[0,0]}$ & $\boldsymbol{[0,0]}$ & $\boldsymbol{[0,0]}$ & $\boldsymbol{[0,0]}$& $\boldsymbol{[0,0]}$ & $\boldsymbol{[0,0]}$& $\boldsymbol{[0,0]}$ & $\boldsymbol{[0,0]}$& $\boldsymbol{[0,0]}$  \\
IQN & $\boldsymbol{[0,0]}$ & $[0, 1416]$ & $[75, 3007]$ & $[0, 834]$ & $[1673, 104689]$ & $[0, 4058]$ & $[57554,87096]$ & $\boldsymbol{[0,0]}$ & $[281, 5935]$ \\
IQN-P & $\boldsymbol{[0,0]}$ & $\boldsymbol{[0,0]}$ & $[11, 2583]$ & $[0, 782]$ & $ [1072,123594]$ & $[0,2539]$ & $[42461,77603]$ & $\boldsymbol{[0,0]}$ & $ [22,2852]$\\
IQN-D & $[0,55]$ & $[0,1541]$ & $[733,8259]$ & $[0, 3959]$ & $[6246, 263553]$ & $[175, 72813]$ & $[68210,113867]$ & $ [0,2131]$ & $[931,7632]$ \\
PCDN & $\boldsymbol{[0,0]}$ & $\boldsymbol{[0,0]}$ & $\boldsymbol{[0,0]}$ & $\boldsymbol{[0,0]}$& $\boldsymbol{[0,0]}$ & $\boldsymbol{[0,0]}$& $\boldsymbol{[0,0]}$ & $\boldsymbol{[0,0]}$& $\boldsymbol{[0,0]}$  \\
NAM & $\boldsymbol{[0,0]}$ & $[0,1201]$ & $[2152, 22760]$ & $\boldsymbol{[0,0]}$ & $ [0,0]$ & $ [0,433]$ & $[2090,11409]$ & $\boldsymbol{[0,0]}$ & $[150, 2630]$ \\
Ours-$q_0$  & $\boldsymbol{[0,0]}^{*}$ & $\boldsymbol{[0,0]}^{*}$ & $\boldsymbol{[0,0]}^{*}$ & $\boldsymbol{[0,0]}^{*}$& $\boldsymbol{[0,0]}^{*}$ & $\boldsymbol{[0,0]}^{*}$& $\boldsymbol{[0,0]}^{*}$ & $\boldsymbol{[0,0]}^{*}$ & $\boldsymbol{[0,0]}^{*}$\\ 
Ours-Mean & $\boldsymbol{[0,0]}^{*}$ & $\boldsymbol{[0,0]}^{*}$ & $\boldsymbol{[0,0]}^{*}$ & $\boldsymbol{[0,0]}^{*}$& $\boldsymbol{[0,0]}^{*}$ & $\boldsymbol{[0,0]}^{*}$& $\boldsymbol{[0,0]}^{*}$ & $\boldsymbol{[0,0]}^{*}$ & $\boldsymbol{[0,0]}^{*}$
\end{tabular}}
\end{center}
\label{tab:CCs}
\end{table*}
\begin{table*}[t]\vspace{-.5em}
\setlength\tabcolsep{1.5pt}
\caption{ Comparison of the Log-Likelihood sum of the test set over all of the train-test folds proposed in \citet{Hern} for all the QR-based models. The bigger the better.}
\begin{center}
\scalebox{0.68}{
\begin{tabular}{r|l|l|l|l|l|l|l|l|l|l}
\multicolumn{1}{r}{}
& \multicolumn{1}{l}{\hfil Housing}
& \multicolumn{1}{l}{\hfil Concrete}
& \multicolumn{1}{l}{\hfil Energy}
& \multicolumn{1}{l}{\hfil Kin8nm}
& \multicolumn{1}{l}{\hfil Naval}
& \multicolumn{1}{l}{\hfil Power}
& \multicolumn{1}{l}{\hfil Protein}
& \multicolumn{1}{l}{\hfil Wine}\\ \cline{2-9}
IQN & $ -703.2\pm 228.4$ & $ -1589.0\pm 228.3$ & $ -1092.7\pm 127.9$ & $ -9819.8\pm 398.4$ & $ -6358.5\pm 1343.0$ & $ -6499.3\pm 633.4$ & $ -27923.5\pm 614.1$ & $ -2040.5\pm 277.3$ \\
IQN-P & $ -613.2\pm 164.3$ & $ -1202.6\pm 157.1$ & $ -1012.3\pm 106.5$ & $ -9579.4\pm 699.5$ & $ -7868.2\pm 2072.1$ & $ -6320.3\pm 819.7$ & $ -29960.0\pm 2204.8$ & $ -2094.6\pm 318.8$ \\
IQN-D & $ -1156.7\pm 178.7$ & $ -2515.7\pm 192.4$ & $ -1839.4\pm 156.0$ & $ -18766.6\pm 1116.2$ & $ -18329.8\pm 5180.0$ & $ -18449.6\pm 1000.6$ & $ -90558.1\pm 5334.5$ & $ -3990.9\pm 309.9$ \\
PCDN & $ -745.4\pm 139.7$ & $ -1398.7\pm 226.4$ & $ -1133.6\pm 141.8$ & $ -9688.2\pm 540.5$ & $ -5664.5\pm 128.6$ & $ -8317.0\pm 343.6$ & $ -36599.9\pm 2961.3$ & $ -1898.7\pm 253.1$ \\
NAM & $ -697.4\pm 216.0$ & $ -1571.2\pm 202.7$ & $ -1461.7\pm 122.3$ & $ -9749.8\pm 483.5$ & $ -5372.5\pm 93.1$ & $ -8503.4\pm 250.5$ & $ -42807.7\pm 3063.3$ & $ -1928.7\pm 239.2$ \\
Ours-$q_0$ & $ \boldsymbol{-421.9\pm 114.2}$ & $\boldsymbol{-852.8\pm 148.9}$ & $\boldsymbol{-722.3\pm 112.4}$ & $\boldsymbol{-5355.8\pm 305.3}$ & $ \boldsymbol{-4033.0\pm 139.8}$ & $\boldsymbol{-4521.1\pm 155.9}$ & $ \boldsymbol{-25600.9\pm 725.1}$ & $\boldsymbol{-1201.6\pm 157.4}$\\
Ours-Mean & $ \boldsymbol{-407.9\pm 111.9}$ & $ \boldsymbol{-808.7\pm 122.8}$ & $ \boldsymbol{-736.4\pm 103.7}$ & $ \boldsymbol{-5309.6\pm 346.4}$ & $ \boldsymbol{-4033.5\pm 102.5}$ & $ \boldsymbol{-4494.5\pm 145.8}$ & $ \boldsymbol{-25596.0\pm 924.6}$ & $ \boldsymbol{-1239.1\pm 154.2}$
\end{tabular}}
\end{center}
\label{tab:UCIlogs}
\end{table*}

The closest neural network method in the literature to the proposed method is the following. Differently, this method approximates the non-partially derivative, which in consequence cannot be applied to approximate a quantile function as we desire, but it has several similarities in the way the derivative is approximated as we will explain hereafter. 

Generically, this method is a deep learning approach to building a monotonic function $H\colon \mathbb{R}^D \rightarrow \mathbb{R}$, called the Unconstrained Monotonic Neural Network (UMNN) proposed in \citep{Wehenkel2019}. The UMNN estimates the derivative of that function as


\begin{equation}
    \label{eq.integ-UMNN}
    H(z) = H(0) + \int _0^z h(t)\, dt,
\end{equation}
where the integral in Eq.~\ref{eq.integ-UMNN} is approximated using the Clenshaw-Curtis quadrature \citep{ClenshawC1960}.

Furthermore, this derivative can be approximated by means of a neural network $\hat{h} \colon \mathbb{R}^D \rightarrow \mathbb{R}_{+}$, whose output is restricted to strictly positive values. Therefore, when $h(z) \coloneq \frac{\partial H}{\partial z}(z) > 0$, the $H(z)$ behaves as a monotone function with respect to $z$, if $\hat h(z) = h(z)$.

The main novelty of our method, compared to the UMNN, is that we build a partial monotonic function, which allows us to approximate a quantile function. Further, the non-dependence of roots regarding $\tau$ (as it can be seen comparing Eq.~\ref{eq.nodes-values} and Eq.~\ref{eq.nodes-values-II}) ensures that, in these points, the quantile function is always partial monotonic (according to Section~\ref{subsect:Rate}).

On the whole, our goal is to show how the proposed solution solves the crossing quantile phenomenon while maintains (or even improves) the performance with respect to other existing alternatives. This can be extended to other fields where crossing quantiles is critical such as Reinforcement Learning (RL) \cite{Dabney_Rowland_Bellemare_Munos_2018,dabney18a,Yang_NEURIPS2019,Zhou_NEURIPS2020} as future work.

\section{EXPERIMENTS}\label{sec:exp}

In this section, we show the performance of the proposed model compared to the non-always monotonic approaches such as the IQN or the NAM. Furthermore, in the Appendix section, an always monotonous model based on restricting the weights - denoted as PCDN - is proposed and added too in the comparison. The goal is to verify that imposing the monotonicity constrain does not lead to a clearly worse predictive model.

In short, the different approaches analysed are:

\paragraph{The deep heteroscedastic Normal distribution (N)} The conditional normal distribution will be considered. Given that this quantiles are calculated from the parametric conditional distribution as $F(\tau, \boldsymbol{x})=\mu(\boldsymbol{x})+\sigma(\boldsymbol{x}) \sqrt{2} \cdot \text{erf}^{-1}( 2\tau - 1), \tau\in\left(0,1\right)$ where $\tau$ is the desired quantile and $\text{erf}^{-1}$, they will not cross. We will denote this approach as \textit{N}.

\paragraph{The Implicit Quantile Network (IQN)} Following Section~\ref{sec:introduction}, we are going to consider the IQN model. In particular, \citet{tagasovska2018frequentist} proposes to add a regularization term with respect to the derivative as $\max\left(-\frac{\partial \phi _w(\boldsymbol{x},\tau)}{\partial \tau}\right)$ to alleviate the crossing quantile phenomena. We will note this alternative optimisation process as \textit{IQN-D}. On the other hand, we will consider another regularisation term that penalises quantiles when crossing quantile phenomena appears as $\max(0.,\phi _w(\tau_1;\boldsymbol{x})-\phi _w(\tau_2;\boldsymbol{x}))$ for each $\tau_1<\tau_2$. We will refer to this approach as \textit{IQN-P}.

\paragraph{The Partial Constrained Dense Network (PCDN)} The main idea of that model is to force a \textit{selected} subset of \textit{weights} to be only positive combined with only considering ReLU activation functions \citep{glorot2011deep}. As it is described in the Appendix Section, the corresponding selected neurons will be increasing functions and can be used to produce a quantile function that is always partial monotonic.

\paragraph{The Not Always Monotonic model (NAM)} Similarly to our main proposed model, NAM has two different functions to optimise: The $\phi _w$ and the $K$. In that case, $K$ corresponds to the $q_0$ value as a general shift of all the quantiles. The predicted quantile function does not ensure that always will avoid crossing quantile phenomenon but will be added to the comparison to analyse their difference.

\paragraph{Our model (Our)} The main proposal of the article can be defined as a partial monotonous function. All alternative selections of $C_0$ described in Section~\ref{subsec:ConstantChoice} are considered for the following comparisons.

\subsection{Data Sets and Experiment Settings}

All experiments are implemented in TensorFlow \citep{tensorflow2015-whitepaper} and Keras \citep{chollet2019keras}, running in a workstation with Titan X (Pascal) GPU and GeForce RTX 2080 GPU. To ensure the strictly positive output values of $\phi _w$ in the proposed model, the final output will have a softplus function \citep{zheng2015improving} with a certain shift, in particular $\phi _w(\tau; \boldsymbol{x})=10^{-3} + \texttt{softplus}(NN(\tau; \boldsymbol{x})+10^{-5})$ where $NN(\tau; \boldsymbol{x})$ is the output of the neural network.  In addition, according to the notation of Section \ref{PCDN}, to force the \textit{selected} weights to be positive we will apply to it a ReLU function \citep{glorot2011deep} to their values. All internal activation functions will be ReLU. Furthermore, all experiments will be trained using an early stopping training policy with $200$ epochs of patience for all compared methods. All the details of the data sets used are detailed in the Appendix Section.

\subsection{Experimental Results}

Quantile regression models fit the conditional density distribution $p(Y\mid\boldsymbol{X})$ in a discrete manner. Thus, we want to study the performance between the different models in order to produce that distribution. However, not all models proposed in this work ensure monotonicity. Therefore, we will decide a common procedure to generate the forecasted likelihood for all of them to perform the following comparison. We calculate the number of predictive quantiles that fall within a discretization of the response space to predict for a thousand of equidistant $\tau\in(0,1)$ points (as it is detailed in the Appendix section).

\paragraph{Evaluating the Number of Crossing Quantiles} In Table \ref{tab:CCs} we show the minimum and maximum number of crossing quantiles taking into account all the folds. We can see that effectively, the normal distribution does not have crossing quantiles due to it comes from the quantile function formula. The PCDN never has any crossing quantile as stated in Appendix section \ref{PCDN}. Finally, as it is shown in Section~\ref{subsec:Mono}, all the main proposed models (Our-$q_0$ and Our-Mean) do not have any crossing quantiles in the root quantile values.

\paragraph{Log-Likelihood Estimation} In Table \ref{tab:UCIlogs}, we compare the log-likelihood adaptation for all the QR-based presented models for the eight different UCI problems. Each position is reporting the mean and standard deviation over the $20$ splits previously defined in \citet{Hern}. We observe that the first five options are far from our proposed model in terms of likelihood, that is the best. Additionally, we show that PCDN and NAM are not significantly worst than IQN. However, we need to take into consideration that PCDN and our proposed model are the only presented options that we can be sure that they are monotonic functions when are evaluated in the desired root quantiles.

\section{CONCLUSIONS}
\vspace{-1em}
Quantile regression is an approach to estimate the conditional density distribution of the response variable in a discrete manner. However, when a single model tries to approximate several quantiles at once the order between them matters. When they are not predicted in an increasing way, then the predicted distribution is invalid. This phenomenon is called ``crossing quantile''. In this work, we introduce a method that solves this issue as shown in Figure~\ref{fig:InitialFigure}.

In particular, the proposed method uses a deep learning model to approximate the partial derivative of an ending quantile function with respect to the quantile input. This partial derivative is imposed to be positive, thus, the ending function will be partially monotonic as desired. In Section \ref{sec:imposepositive}, we described how to calculate this derivative by using a Chebyshev Polynomial approximation ensuring the monotonicity of the quantile function up to any desired precision.

Having computationally guarantees of the monotonicity of these models, our final goal was to verify if the imposed restrictions would adversely affect the performance of the model compared with a non-ensured quantile regression model. As we saw in Section \ref{sec:exp}, the proposed models outperforms in terms of likelihood estimation and yields non-quantile crossing forecasts with respect the other compared models.

The proposed model constitutes a generic deep learning wrapper for any architecture to build a partial monotonic quantile function avoiding the crossing quantile phenomenon.

\section*{Acknowledgement}

We gratefully acknowledge the Industrial PhD Plan of Generalitat de Catalunya with BBVA Data and Analytics for funding this research. J.G. has been supported by NextGenerationEU within the Spanish
national Recovery, Transformation and Resilience plan. The UB recognizes that part of the research described in this chapter was partially funded by TIN2015-66951-C2, SGR 1219.

\bibliography{named1325}

\appendix

\clearpage

\section{PARTIAL CONSTRAINED DENSE NETWORK} \label{PCDN}

PCDN has been used as reference method to be compared in the experiments, Section~\ref{sec:exp}. Although it is not in the literature, PCDN is a natural approach to be considered in order to address the problem of crossing quantiles because it considers a classical approach of constraining the weights and activation functions to ensure that fully-connected neural network layers are monotonic. As an extension to this, PCDN constitutes a combination of partial monotone layers to ensure that the global model is partially monotone with respect to the quantile input as we will detail hereafter.

\begin{figure}[h]
    \centering
    \includegraphics[scale=.6]{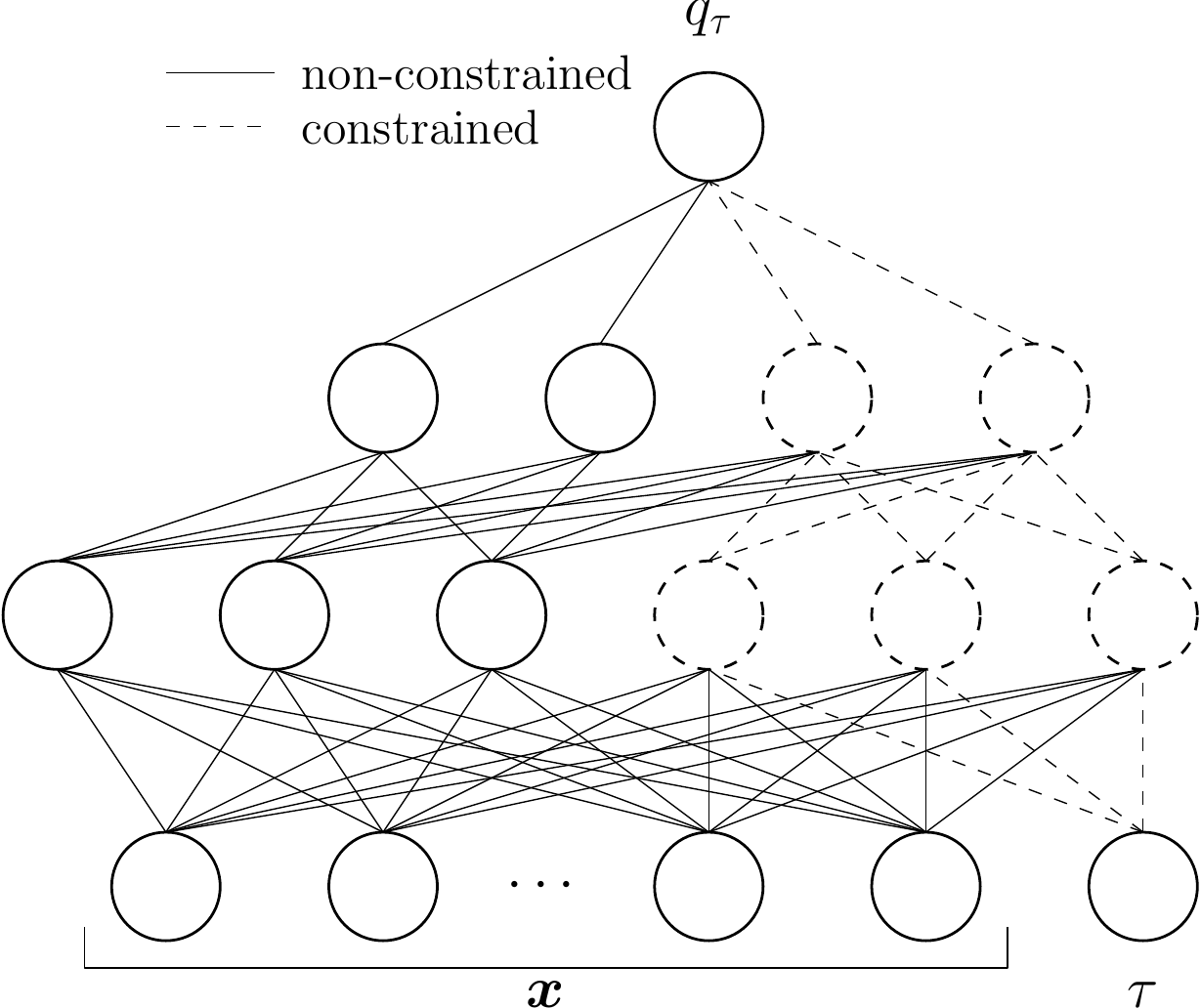}
    \caption{\small PCDN scheme.}
    \label{fig:pcdn}
\end{figure}

The construction of neural networks with a monotonously increasing function of the inputs is a problem that has been identified in the literature, see, for instance, \citet{sill1998monotonic}. These solutions were designed for completely dense models (i.e. those where all the neural network layers are fully-connected). These approaches were based on the following principle: a dense model with positive weights and with monotonically increasing activation functions gives a monotonously increasing function with respect to the input. Although, it is a very easy implementation to carry out, it has been found that, at the practical level, the impossibility of having weights of different signs makes the learning process complicated \citet{Wehenkel2019}. 

According to the notation in the Section~\ref{sec:introduction}, to solve the crossing quantile phenomenon it is enough to define a monotonic function $f$ with respect to the input $\tau$ (that corresponds to the quantile value). 

To built the model proposed in this section, we start with an IQN model, $\psi \colon [0,1] \times \mathbb{R}^D \rightarrow \mathbb{R}$, with only fully-connected layers by default. 

Then, we select a desired part of the neurons for each hidden layer to ensure that $\psi$ is monotonic with respect to $\tau$; for instance, $50\%$ of the neurons in each layer. 

Next, we force that these selected neurons only have positive weights, their activation function are monotonically increasing, and they are only connected with the previous layer neurons that were also selected. 

Finally, in the first layer, we only connect the weights that start from the input $\tau$ to the previously selected neurons of the following hidden layer as it is shown in Figure~\ref{fig:pcdn}. We refer to these kind of models as Partial Constrained Dense Network (PCDN). 

On the one hand, the inputs $\boldsymbol{x}$ are connected to the output with restricted and non-restricted weights. 
On the other hand, the $\tau$ value is connected with the output of the model only with restricted neurons to ensure it is a partial monotonic increasing function with respect to $\tau$. By applying PCDN we restrict the number of constrained weights and neurons to only selected part of the neural network. This model could be considered as an extension of other proposed models in the literature such as Monotone Composite Quantile Regression Neural Network (MCQRNN) \citep{cannon2018non}, which predicts a fixed number of quantiles.

\section{DETAILS OF THE DATA SETS ON THE EXPERIMENTS} \label{sec:experiments}

In this section, we show the performance of the proposed models compared to the baseline IQN. The goal is to verify that imposing the monotonicity of the PCDN or the main proposed model does not lead to a clearly worse predictive model.

\paragraph{Synthetic Glasses Data set} The points in the back of the curves of Figure~ \ref{fig:InitialFigure} corresponds to the mixture of two distributions. 
On the one hand, a noisy evaluation of $5*\sin(x_i)+0.5 + \epsilon$ is done where each $x_i$ corresponds to evaluate $3000$ equidistant points such that $0 < x_i < 3\pi$. The random noise correspond to a $\epsilon\sim\text{Beta}{\left(\alpha=0.5, \beta=1\right)}$. 
On the other hand, the second distribution of the mixture is the evaluation of $3000$ points such that $\pi < x_j < 4\pi$ into $5*\sin(x_j)+0.5+\varepsilon$. This time, the noise is defined by $\varepsilon\sim-\text{Beta}{\left(\alpha=0.5, \beta=1\right)}$. 
After that, all these values are normalized taking into account the maximum value of them. The regression problem consists in predicting, given an assigned value $x\in[0,1]$, the corresponding $y\in\mathbb{R}$ generated as explained. These data have multi-modalities to encourage crossings of the predicted quantiles. The $50\%$ generated data was considered as test data, $40\%$ for training, and $10\%$ for validation.

The neural network architecture used for the $\phi _w$ of IQN and the $\phi _w$ and $K _w$ of our method consists of $4$ dense layers with output dimensions $120$, $60$, $10$, and $1$ respectively. Regarding the training time, all models took less than $4$ minutes to converge.

\paragraph{UCI data sets} In order to satisfy the goal of this section of checking the performance of IQN, PCDN, NAM and the main proposed models, we applied these models to $8$ different UCI Machine learning data sets \citep{Dua:2019}. These data sets are commonly used for various regression tasks. In particular, we used the $20$ splits proposed in \citet{Hern} and widely used in later works \citep{Gal2015DropoutB, lakshminarayanan2017simple}.

Regarding the trained models, all of them share the same architecture for all the models: a single dense hidden layer of $200$ neurons. However, as NAM and the main proposed model requires the architecture for $\phi _w$ and $K _w$, we decided to assign the half part of the neurons of the single hidden layer to each of them.

In addition, we will analyse the number of crossing quantiles that are obtained for each test set considering the $980$ quantiles $[ 0.010, 0.011, 0.012, \dotsc, 0.990 ]$.

\end{document}